\documentclass{article}
\usepackage{ijcai21}

\usepackage{times}
\usepackage{soul}
\usepackage{url}
\usepackage[hidelinks]{hyperref}
\usepackage[utf8]{inputenc}
\usepackage[small]{caption}
\usepackage{graphicx}
\usepackage{booktabs}
\usepackage{algorithm}
\usepackage{algorithmic}
\usepackage{amsmath}
\usepackage{amsthm}
\usepackage{amssymb}
\usepackage{subcaption}
\usepackage{indentfirst}
\usepackage{titlesec}
\newtheorem{definition}{\textbf{Definition}}\newtheorem{theorem}{\textbf{Theorem}}\newtheorem{proposition}{\textbf{Proposition}}

\makeatother

\title{BN-invariant Sharpness Regularizes the Training Model to Better Generalization}

\author{
Mingyang Yi$^{1,2}$
\and
Huishuai Zhang$^3$\and
Wei Chen$^3$\and
Zhi-ming Ma$^{1,2}$\And
Tie-Yan Liu$^{3}$
\affiliations
$^1$University of Chinese Academy of Sciences\\
$^2$Academy of Mathematics and Systems Science, Chinese Academy of Sciences\\
$^3$ Microsoft Research
\emails
yimingyang17@mails.ucas.edu.cn,
mazm@amt.ac.cn,
\{huzhang,wche, tie-yan.liu\}@microsoft.com
}
\hypersetup{draft}

\begin{document}

\maketitle

\begin{abstract}
It is arguably believed that flatter minima can generalize better. However, it has been pointed out that the usual definitions of sharpness, which consider either the maxima or the integral of loss over a $\delta$ ball of parameters around minima, cannot give consistent measurement for scale invariant neural networks, e.g., networks with batch normalization layer. In this paper, we first propose a measure of sharpness, BN-Sharpness, which gives consistent value for equivalent networks under BN. 
It achieves the property of scale invariance by connecting the integral diameter with the scale of parameter. Then we present a computation-efficient way to calculate the BN-sharpness approximately i.e., one dimensional integral along the "sharpest" direction. Furthermore, we use the BN-sharpness to regularize the training and design an algorithm to minimize the new regularized objective. Our algorithm achieves considerably better performance than vanilla SGD over various  experiment settings.
\end{abstract}

\section{Introduction}
\noindent With the support of big data, deep learning techniques has achieved a huge success across domains. 
In recent years, it becomes a common sense that deep neural networks have perfect training performance. Specifically, \cite{zhang} empirically shows that the popular neural networks (NN) can always reach zero training error at the end of training process. Moreover, \cite{Du} and \cite{Zhu} prove that, for the neural networks with sufficient width, gradient descent converges to the global minima of the objective, under some conditions on the data generation. Thus how to find minima taht generalize well seems to be more critical in the study of deep neural network. 
\par
It's arguably believed that minima locate in a flat valley generalize better. \cite{BehnamNeyshabur} employs the PAC Bayes theory \cite{PAC1} and proves a generalization bound give by "expected sharpness". The conclusion gives us a new angle to discuss generalization error of a neural network. This viewpoint is also empirically confirmed by \cite{SharpMinima} and \cite{Li}. However, for neural network with scale invariant property, the definitions of sharpness become problematic.
\par
One important case is deep neural network with batch normalization (BN) \cite{Sergey}. BN is an important component which normalize the distribution of input to each neuron in the network by mean and standard deviation of input computed over a mini-batch of training data. 
Batch normalization will bring a scale invariant property to neural network which will make the definition of sharpness for neural network problematic (two parameter points with same generalization error but different sharpness). Since sharpness is related to generalization, a natural question is can we develop an appropriate definition of sharpness for neural network with batch normalization? Meanwhile, can we leverage the appropriate measurement of sharpness to develop an algorithm which helps us find minima generalize better? 
\par
In this paper, we answer the two above problems positively. First, we show that the original definition of sharpness for neural network with batch normalization is ambiguous. Specifically, the original sharpness ($\delta$-sharpness) consider the maxima of loss over a $\delta$ ball of parameters around minima with radius $\delta$ unrelated to parameter $\theta$. However, due to the scale invariant property of neural network with batch normalization, the radius $\delta$ should scale with parameter norm $\theta$. On the other hand, the new measurement of sharpness should be computational efficient since we want to leverage it to help us find minima generalize better. Based on the two above points, we propose a "BN-sharpness" to measure the sharpness for BN neural network. In specific, BN-Sharpness is a one dimensional integral of the loss's difference value along the "sharpest" direction in a small neighborhood while the diameter of this neighborhood is related with parameter scale. It's an intuitively thought that a minimum locate in a flat valley when the loss values in the valley are closely with each other along the steepest changing direction.
\par
In order to find minima locate in flat valley which generalize better, we penalize the original optimization objective with BN-Sharpness. Due to the one dimensional integral structure of BN-Sharpness, we can easily acquire the "sharpest direction" by optimization on manifold \cite{Boumal} as well as the gradient of BN-Sharpness.
\par 
It's empirically showed that SGD with large batch size tend to produce "sharp" minima which generalize worse \cite{SharpMinima}. We test our algorithm on CIFAR dataset \cite{cifar}, it has preferable results under large batch size compared with baselines (SGD, Entropy SGD).
\subsection{Related work}
\noindent
Sharpness is a property connected with the loss surface. \cite{Li} first visualize the loss landscape of NN to describe that "sharp" minima indeed generalize worse. Such phenomenon was also empirically confirmed by \cite{SharpMinima} and \cite{Wu}. \cite{BehnamNeyshabur} also gives the conclusion a theoretical interpretation.
\par
In fact, the definition of sharpness is proved to be ambiguous for ReLU neural network \cite{SharpMinimacan}. Suffering from the definition of sharpness and computational complexity, some local structured algorithms aimed to produce "flat" minima avoid exactly definition of sharpness. Entropy SGD proposed by \cite{entropy_sgd} is motivated by the local geometry of the energy landscape, it alters loss as accumulated "energy value" instead of single point value. \cite{smooth_out} argues that choosing integral of a small cubic around iterate as loss function can produce a smoother loss function which will lead to a "flat" minimum. \cite{HuanWang} sets an expected loss function to produce "flatter" minima under PAC Bayes theory. However, all of these algorithms will face a high dimensional integral.
\par
There are lots of sharpness based algorithms to find flat minima \cite{entropy_sgd,HuanWang,smooth_out}. Since it's hard to compute the gradient of the original sharpness, they set optimization objective as integral of the loss function in a ball to force iterates walking into a flat valley instead of optimizing original loss function regularized by sharpness. This kind of optimization objective involves a high dimensional integral which is actually hard to be computed in practice.
\par
For the topic of neural network with batch normalization, \cite{BN2} and \cite{HuangLei} give Riemannian approach to optimize neural network with batch normalization. \cite{Yuan} propose a general variant of batch normalization to speed up training process. 
However, none of these works consider combining the sharpness with batch normalization to reduce generalization error.
\subsection{Contribution}
\setlength{\parindent}{2em}(1) We show that the generally accepted definition of sharpness: $\delta$-sharpness \cite{SharpMinima} is ill-posed for NN with BN. 
We propose a scale invariant and computational efficient "BN-Sharpness" to measure sharpness for NN with BN.
\par
(2) We present a computational efficient algorithm based on BN-Sharpness to enhance generalization ability of iterates which has preferable experimental results. 

\section{Background}
\subsection{Definitions of sharpness}
There are lots of definitions to describe sharpness \cite{SharpMinima,three_factors,Wu}. We focus on a widely discussed definition: $\delta$-sharpness this paper.
\par
\begin{definition}[$\delta$-$L^{p}$ sharpness]
	Let $B_{2}(\theta,\delta)$ be an Euclidean ball centered on a minimum $\theta$ with radius $\delta$. Then, for a non-negative valued loss function $L(\cdot)$ and non-negative number $p$, the $\delta$-$L^{p}$ sharpness will be defined as
	\begin{equation}\small
	S^{p}_{\delta\text{-sharpness}}(\theta)=\frac{\left(\int_{\theta^{\prime}\in B_{2}(\theta,\delta)}\left|L(\theta^{\prime})-L(\theta)\right|^{p}d\theta^{\prime}\right)^{\frac{1}{p}}}{1+L(\theta)}.
	\label{sharpness}
	\end{equation}
	\label{def:sharpness}
\end{definition}
In fact, denominator of equation (\ref{sharpness}) will close to 1 when $L(\theta)$ is small. Thus, we ignore it during discussion. We can actually prove that $\delta$-sharpness $S_{\delta\text{-sharpness}}$ used in \cite{SharpMinima} is actually $\delta$-$L^{\infty}$ sharpness\footnote{A detailed discussion can be referred to a long version of this paper which can be found in .}. There are some other measures of sharpness such as $\text{tr}(H)$ \cite{three_factors} and $\lambda_{\text{max}}(H)$ \cite{Wu} where $H$ is the Hessian of minimum. $\lambda_{\max}(\cdot)$ is the largest eigenvalue.
\subsection{Batch normalization}
\noindent
We briefly revisit the batch normalization and its properties. A network with BN transforms the input value to a neuron from $z={'}{T}x$ to
\begin{equation}\small
BN(z)=\gamma \frac{z-E(z)}{\sqrt{Var(z)}}+\beta = \gamma \frac{\theta^{T}(x-E(x))}{\sqrt{\theta^{T}V_{x}\theta}}+\beta.
\end{equation}
We can easily verify $BN(az)=BN(z)$ for any $a>0$. Then, for a partially batch normalized neural network (Some of input values to neuron are batch normalized) with $N$ parameter vector $\theta=(\theta_{1},\cdots,\theta_{N_{1}},\theta_{N_{1}+1},\cdots,\theta_{N})^{T}$, where $\theta_{i}$ ($i=1\cdots N_{1}$) is parameter vector connected with batch normalized neuron and $(\theta_{N_{1}+1},\cdots,\theta_{N})$ connect neuron without batch normalization. We can actually view the affine parameters $\gamma,\beta$ and bias parameters of each neuron as un-batch normalized parameters. We integrate them into parameter vector $\theta$ during our latter discussion. In the latter discussion, our analysis is applied to partially batch normalized neural network if we don't give an additional illustration. Now we give the definition of scale transformation for partially batch normalized neural network.
\begin{definition}
$T_{\vec{a}}(\cdot)$ denote scale transformation for a partially batch normalized neural network with  
\begin{equation}\small
T_{\vec{a}}(\theta)=(a_{1}\theta_{1},\cdots a_{N_{1}}\theta_{N_{1}},\theta_{N_{1}+1},\cdots,\theta_{N})^{T}, a_{i}>0.
\end{equation}
\end{definition}
We see the loss function $L(\theta)$ satisfy $L(\theta)=L(T_{\vec{a}}(\theta))$ for any $T_{\vec{a}}(\cdot)$ \footnote{For the situation of Pre-ResNet which has input $BN((\vec{1}+\theta)^{T}z)$ in skip connection layer. It's not a scale invariant layer, we treat these $\theta$ as un-batch normalized parameters.}. We call this the property of scale invariant of BN. In fact, the ambiguous of sharpness is brought by such property.
\section{BN-Sharpness}
\noindent
In this section, we give a $\delta$-BN sharpness to measure sharpness. A well defined geometrical measurement linked with generalization error for neural network with batch normalization should be scale invariant. We start with theorem to explain why the original measurement of sharpness is inappropriate.

\begin{theorem}
	Given a partially batch normalized network, $\delta$-sharpness is not scale invariant.
	\label{pro:invariant}
\end{theorem}
\begin{proof}
	\emph{For $S_{\delta\text{-sharpness}}(\cdot)$, We consider a partially batch normalized neural network, without of loss generality we assume the network has batch normalized neuron in one layer. Then we choose $\vec{a}$ as $(1,\cdots,a_{0},\cdots,a_{0},1,\cdots,1)^{T}$ where $a_{0}$ locate in the same coordinate with batch normalized parameter.} 
	\emph{For any $\delta>0$, we can choose 
	\begin{equation}\small\small\small
	0< a_{0}\leq \frac{\delta}{\sqrt{N}\max_{1\leq i\leq N}\|\theta_{i}\|}
	\end{equation} 
	with $L(T_{\vec{a}}(\theta))=L(\theta)$. Because a parameter $\theta$ with one layer zero weight matrix locate in $B_{2}(T_{\vec{a}}(\theta),\delta)$ which result in $\mathop{\max}_{\theta^{\prime}\in B_{2}(\theta,\delta)}\left(L(\theta^{\prime})-L(\theta)\right)$ is at least as high as constant function. Then, the value of $S_{\delta\text{-sharpness}}(T_{\vec{a}}(\theta))$ is relatively large.}
\end{proof}
\par
The above result reveals that the original definition of sharpness is not well defined for BN-network. Actually, we can achieve a minimum with small generalization error but infinite sharpness if we rescale the minimizer adequately close to zero for $\delta$-sharpness. 
Hence, using $\delta$-sharpness as a measurement of generalization is meaningless. We see the ill-posed issue of $\delta$-sharpness suffers from the scale invariant ability of batch normalization network. Based on Theorem \ref{pro:invariant}, we aim to derive a \emph{scale invariant} sharpness. 
\par
Due to these, we present a scale invariant meanwhile computational efficient measurement: "BN-Sharpness", which is instructive for us to leverage it to find "flat" minima. Now, we give the exact definition of "BN-Sharpness".
\par
\begin{definition}[$\delta$-$L^{p}$ BN-Sharpness]
Given a positive number $\delta$, and a parameter point $\theta$, and a partially batch normalized network, the $\delta$-$L^{p}$ BN-Sharpness is defined as
	\begin{equation}\small\small
	\|L(\cdot)\|_{p}^{\delta,\theta}=\mathop{\sup}_{v\in\phi(\theta)}\frac{1}{\delta^{\frac{1}{p}}}\left(\int_{-\delta}^{\delta}\left|\frac{L(\theta+tv)-L(\theta)}{\delta}\right|^{p}dt\right)^{\frac{1}{p}},
	\label{eq:Lpsharpness}
	\end{equation}
	\label{def:Lpsharpness}
where $\phi(\theta)$ is a set composed by $v$ with $\|v_{i}\|=\|\theta_{i}\|,\sqrt{\sum_{j=N_1+1}^{N}\|v_{j}\|^{2}}=1;i=1,\cdots,N_{1}$.
\end{definition}
We notice $v\in\phi(\theta)$ has same dimension with $\theta$. And it's decided by the parameter of a partially batch normalized network. Since $v\in\phi(\theta)$ has identically $l_{2}$ norm which is $\sqrt{\sum_{i=1}^{N_{1}}\|\theta_{i}\|^{2}+1}$, we use $\|\phi(\theta)\|$ to denote it. 
For simplicity, we use BN-Sharpness to substitute $\delta$-$L^{p}$ BN-Sharpness. 
\par
As we have discussed in Section 2, the $L^{p}$ norm of function $L(\cdot) - L(\theta)$ defined in $B_{2}(\theta,\delta)$ can be a measurement of sharpness ($\delta$-sharpness is $L^{\infty}$ norm). The crucial problem of such $L^{p}$-sharpness is losing sight of parameter scale when choosing region diameter $\delta$. 
On the other hand, it's a high dimensional integral ($p<\infty$) which is hard to be computed in practice let alone combining them to reduce sharpness. The two shortages are also hold when $p=\infty$.
\par
We see BN-Sharpness consider the scale of batch normalized parameter. Meanwhile BN-Sharpness is an one dimensional integral along the sharpest direction $v$ of $L(\theta)$ for each parameter component $\theta_{i}$. It will be computational efficient especially for the gradient which we will discuss it more specifically in Section 5. By Taylor's expansion, for each parameter component $\theta_{i}$ we have
\begin{equation}\small\small\small
\begin{aligned}
L(\theta+tv)-L(\theta)
&\approx \sum\limits_{i=1}^{N}t\nabla_{\theta} L(\theta)_{i}^{T}v_{i}+o(t\|v\|)\\
&\leq t\sum\limits_{i=1}^{N}\|\nabla_{\theta} L(\theta)_{i}\|\|v_{i}\|+o(t\|\phi(\theta)\|).
\end{aligned}
\label{eq:sharpnest}
\end{equation}
Here $\nabla L(\theta)_{i}$ represent gradient of $L(\cdot)$ to component parameter vector $\theta_{i}$. The equation \eqref{eq:sharpnest} gives an approximate calculating of the "sharpest" direction when $\delta$ is small. Precisely calculation of BN-Sharpness requires optimization on Oblique manifold which we will discuss in Section 4.
\par
Now, we prove that BN-Sharpness is scale invariant. Hence, it's appropriate to measure sharpness for BN neural network. 
\par
\begin{theorem}
	The loss function $L(\theta)$ of a DNN with batch normalization satisfies $\|L(\cdot)\|_{p}^{\delta,\theta} = \|L(\cdot)\|_{p}^{\delta,T_{\vec{a}}(\theta)}$ for any $\vec{a}$ without negative element.
	\label{thm:invariant}
\end{theorem}
\begin{proof}
	\emph{For any $a\neq0$, we notice that $L(\theta)=L(T_{\vec{a}}(\theta))$. Therefore for any $v$ with $v\in\phi(\theta)$, we have
	\begin{equation}\small\small
	L(\theta+tv)-L(\theta)=L(T_{\vec{a}}(\theta)+tT_{\vec{a}}(v))-L(T_{\vec{a}}(\theta)),
	\end{equation}\small
	for $t\in[-\delta,\delta]$. It's easily to verify that
	\begin{equation}
	\begin{aligned}
	&\int_{-\delta}^{\delta}\left|\frac{L(\theta+tv)-L(\theta)}{\delta}\right|^{p}dt\\
	& = \int_{-\delta}^{\delta}\left|\frac{L(T_{\vec{a}}(\theta)+tT_{\vec{a}}(v))-L(T_{\vec{a}}(\theta))}{\delta}\right|^{p}dt
	\end{aligned}
	\end{equation}
	Since
	\begin{equation}\small\small
	\begin{aligned}
	&\mathop{\sup}_{v\in\phi(\theta)}\frac{1}{\delta^{\frac{1}{p}}}\left(\int_{-\delta}^{\delta}\left|\frac{L(T_{\vec{a}}(\theta)+tT_{\vec{a}}(v))-L(T_{\vec{a}}(\theta))}{\delta}\right|^{p}dt\right)^{\frac{1}{p}}\\
	&=\mathop{\sup}_{v\in\phi(T_{\vec{a}}(\theta))}\frac{1}{\delta^{\frac{1}{p}}}\left(\int_{-\delta}^{\delta}\left|\frac{L(T_{\vec{a}}(\theta)+tv)-L(T_{\vec{a}}(\theta))}{\delta}\right|^{p}dt\right)^{\frac{1}{p}}\\
	&=\|L(\cdot)\|_{p}^{\delta,T_{\vec{a}}(\theta)},
	\end{aligned}
	\end{equation}
	we get the conclusion.}
\end{proof}
\par
Actually, we can derive a relationship between BN-Sharpness and generalization error. The result is based on PAC Bayesian theory \cite{PAC1,PAC2}.
\begin{theorem}
	Given a "prior" distribution $P$(for parameter $\theta$) over the hypothesis that is independent of the training data, with probability at least $1-\varepsilon$, we have
	\begin{equation}\small
	\begin{aligned}
	|\mathbb{E}_{u}[L(\theta+u)] -\hat{L}(\theta)|&\leq\delta^{1+\frac{1}{p}}\|L(\cdot)\|_{p}^{\delta,\theta} \\
	&+ 4\sqrt{\frac{1}{m}\left(KL(\theta+u||P)+\log{\frac{2m}{\varepsilon}}\right)},
	\end{aligned}
	\label{eq:generalization_bound}
	\end{equation}
	where $L(\cdot)$ and $\hat{L}(\cdot)$ are respectively expected loss and training loss, $m$ is the number of training data and $\theta$ is the parameter learned from training data. As long as $u$ is an uniform distribution on any specific $v\in\delta\cdot\phi(\theta)$.
\end{theorem}
A straightforward connection between generalization error $|L(\theta) - \hat{L}(\theta)|$ and sharpness is nontrivial. We can only derive a perturbed generalization error $|\mathbb{E}_{u}[L(\theta+u)] -\hat{L}(\theta)|$, where $u$ is the perturbation variable. A small perturbation variable $u$ will make the perturbed generalization error close to the real generalization error which involves a small $\delta$ in equation \eqref{eq:generalization_bound}. We claim that the result gives a quantitatively description of generalization error based on BN-Sharpness. It's a direct corollary according to the equation (6) in McAllester \cite{PAC3} and the definition of BN-Sharpness. 
\section{Regularizing Training with BN-Sharpness}
\noindent
We already have an appropriate definition of BN-Sharpness. It has also been confirmed that smaller BN-Sharpness leads to better generalization. We naturally consider leveraging it to find a flat minimum generalizes better. 
\par
Intuitively, we consider using BN-Sharpness as a regularization term tend to produce flat minima for BN neural network. 
We substitute the optimization problem $\mathop{\min}_{\theta} L(\theta)$ with 
\begin{equation}\small
\mathop{\min}_{\theta}L(\theta)+\lambda(\|L(\theta^{\prime})\|_{p}^{\delta,\theta})^{p}.
\label{eq:optimization problem}
\end{equation}
An interesting phenomenon is such regularization term not only computational efficient and appropriate but also well-posed. More specifically, traditional regularization term such as $l_{1},l_{2}$ would change the minimums set of loss function. But BN-Sharpness regularization term keeps minima of original loss function in the minimal point set of regularized loss function, while it force iterates move to a flat valley. The next theorem states that regularization term in \eqref{eq:optimization problem} wouldn't remove minima of original loss function when $\delta$ is small.
\begin{theorem}
	The minima of problem $\mathop{\min}_{\theta}L(\theta)$ are also minima of optimization problem (\ref{eq:optimization problem}), when $\delta\rightarrow0$ and $p\geq 1$.
	\label{thm:optimization_problem}
\end{theorem}
\par
\par
The optimization problem \eqref{eq:optimization problem} is actually a multi-target programming with "accuracy target" and "flatness target". $\lambda$ in \eqref{eq:optimization problem} is actually a proportion between the two terms. Since we aim to find "flat" minima rather than "flat" point, an appropriate proportion between the two purpose is crucial. 
\par
Now, we give a computational efficient algorithm to solve the optimization problem \eqref{eq:optimization problem}. Here we simply explain our algorithm flow. We notice that the obstacle of optimization problem \eqref{eq:optimization problem} is calculating the gradient of regularization term $(\|L(\theta^{\prime})\|_{p}^{\delta,\theta})^{p}$ (BN-Sharpness). It involves two steps: First, calculating the "sharpest" direction which we need optimization on Oblique manifold; Second, computing the gradient of integral term in BN-Sharpness under the "sharpest" direction.
\par
Now we give the complete flow of our algorithm: Algorithm \ref{alg1}. Then we will discuss more details about it.
\begin{algorithm}
	\caption{SGD with BN-Sharpness regularization.}
	\label{alg1}
	\begin{algorithmic}
		\STATE {Input $\delta>0$, even number $p$, initialize point $\theta^{0}\in\mathbb{R}_{n}$, $v_{0}\in\phi(\theta^{0})$ set to be the gradient direction like equation \eqref{eq:sharpnest}, regularization coefficient $\lambda$, iterations $K_{1},K_{2}$ and learning rate $\eta$.}
		\STATE {{\bfseries Optimization with BN-Sharpness term}}
		\STATE {$k_{2}=0$}
		\WHILE {$k_{2}\leq K_{2}$ {\bfseries and} $ \nabla_{\theta} L(\theta^{k_{2}})\neq0$}
		\STATE {{\bfseries Iterate $K_{1}$ times to search the sharpest direction $v^{k_{1}}(\theta^{k_{2}})$ in point $\theta^{k_{2}}$}}
		\STATE {$\theta^{k_{2}+1}=\theta^{k_{2}}-\eta(\nabla_{\theta} L(\theta^{k_{2}})+h_{1}(\theta^{k_{2}},v^{k_{1}}(\theta^{k_{2}}),\delta,p,\lambda))$, $k_{2}=k_{2}+1,v_{0}\in\phi(\theta_{k_{2}})$}
		\ENDWHILE
		\RETURN{$\theta_{k_{2}}$}
	\end{algorithmic}
\end{algorithm}
\par
In Algorithm \ref{alg1}, the first inner loop is calculating the "sharpest" direction which is the process of optimization on Oblique manifold. The second inner loop is general gradient descent to update parameter $\theta$. In addition, we make an approximation to $\lambda \nabla_{\theta}\frac{1}{\delta}\int_{-\delta}^{\delta}\left(\frac{L(\theta+tv)-L(\theta)}{\delta}\right)^{p}dt$ and  $\nabla_{v}\frac{1}{\delta}\int_{-\delta}^{\delta}\left(\frac{L(\theta+tv)-L(\theta)}{\delta}\right)^{p}dt$ in equation \eqref{eq:approximate} and \eqref{eq:gradient approximate} to further reduce computational redundancy. We respectively denote them as $h_{1}(\theta,v,\delta,p,\lambda)$ and $h_{2}(\theta,v,\delta,p)$. 
\par
\subsection{Searching the Sharpest Direction $v$}
\noindent
The extra computation cost comparing to vanilla SGD is the procedure of searching the "sharpest" direction $v$ which is the second inner loop in Algorithm \ref{alg1}. Inspiring by equation (\ref{eq:sharpnest}), we can also choose $v_{k}\in\phi(\theta_{k})$ with $v_{k}^{i}$ has same direction with $\nabla L(\theta_{k})_{i}$ for each step to avoid the searching process (Iterate $K_{1}$ times in Algorithm \ref{alg1}). 
\par
However, there is a more accurate searching procedure which is optimization on Oblique manifold. Searching $v$ in BN-Sharpness (\ref{eq:Lpsharpness}) is equivalent to solving optimization problem:
\begin{equation}\small
\mathop{\arg\max}_{v\in\phi(\theta)} \frac{1}{\delta^{\frac{1}{p}}}\left(\int_{-\delta}^{\delta}\left|\frac{L(\theta+tv)-L(\theta)}{\delta}\right|^{p}dt\right)^{\frac{1}{p}}.
\label{eq:optimization on oblique}
\end{equation}
This is a constrained optimization problem which can be converted to optimization on manifold. We briefly present optimization on manifold here. Detailed introduction can be referred to \cite{Boumal}.
\par
Optimization on manifold converts a constrained optimization problem into an un-constraint problem while iterates locate in a manifold satisfy the constraint. Specific definition of manifold $\mathcal{M}$ can be found in \cite{Absil}. Here we only consider matrix manifold i.e. a subspace of Euclidean space. Solving problem \eqref{eq:optimization on oblique} needs gradient assent on manifold which produce iterates as
\begin{equation}\small
v_{k+1}={\rm Retr}_{v_{k}}\left(\frac{1}{L}{\rm grad}f(v_{k})\right). \label{eq:gdmanifold}
\end{equation}
Here ${\rm grad}f(x)$ is Riemannian gradient which is a map from tangent space $T_{x}\mathcal{M}$ of point $x$ to manifold $\mathcal{M}$. 
The optimization problem \eqref{eq:optimization on oblique} defined on the general Oblique manifold. We need the related formulations in Oblique manifold to process our algorithm.
\begin{definition}[Oblique manifold]
	Oblique manifold is a subset of Euclidean space satisfy ${\rm St}(n,p)=\{X\in\mathbb{R}^{n\times p}:{\rm ddiag}(X^{T}X)=I_{p}\}$, where ${\rm ddiag}(\cdot)$ is diagonal matrix of a matrix.
\end{definition}
\par
Apparently, for a given $\theta$ in equation (\ref{eq:optimization on oblique}), $v$ lives in a special Oblique product manifold $\|\theta\|_{1}{\rm St}(n_{1},1)\times\cdots\times\|\theta\|_{N_{1}}{\rm St}(n_{N_{1}},1)\times{\rm St}(\sum_{j=N_{1}+1}^{N}n_{j},1)$ where $n_{i}$ is the dimension of parameter $\theta_{i}$.
\par
As we discussed, optimization on manifold requires ${\rm Retr}_{x}$ and ${\rm grad}f(x)$. The two items on single Oblique manifold $\|\theta\|{\rm St}(n,1)$ respectively are
\begin{equation}\small
\begin{aligned}
{\rm Retr}^{\|\theta\|}_{x}(\eta)&=\frac{x+\eta}{\|x+\eta\|}\|\theta\|,
P_{x}^{\|\theta\|}(\eta)=\eta-\frac{x}{\|\theta\|^{2}}{\rm ddiag}(x^{T}\eta),\\
{\rm grad}f(x)&=P_{x}^{\|\theta\|}\left(\nabla f(x)\right)=\nabla f(x)-\frac{x}{\|\theta\|^{2}}{\rm ddiag}(x^{T}\nabla f(x))
\end{aligned}
\label{eq:grad}
\end{equation}
where $P_{x}^{\|\theta\|}(\cdot)$ is projection matrix of the tangent space at $x$. These results can be derived from the general formulas in \cite{Absil,Absil1}. Then we can generalize it to achieve the update rule in product manifold we used. The procedure is gradually update $v^{k}_{i}$ according to the manifold it located \cite{BN2}. 
\par
In addition, we note that our algorithm involves 
$ \nabla_{\theta}\frac{1}{\delta}\int_{-\delta}^{\delta}\left(\frac{L(\theta+tv)-L(\theta)}{\delta}\right)^{p}dt$ (The gradient descent for parameter $\theta$) and $\nabla_{v}\frac{1}{\delta}\int_{-\delta}^{\delta}\left(\frac{L(\theta+tv)-L(\theta)}{\delta}\right)^{p}dt$ (Searching the sharpest direction). Precise calculation of the two terms will bring some extra computational redundancy. Because $L(\theta)$ usually be a neural network, but calculating the integral term requires sampling value of $L(\theta+tv)$ and $\nabla L(\theta + tv)$. Each sampling corresponds to a forward and backward process of neural network.
The next proposition gives an approximation to the gradient of BN-Sharpness. It reduces the times of froward and backward process to two. In the next proposition, we aim to a specific $v$ and suppose that $v$ is a constant vector.
\begin{proposition}
	Given $\delta>0$, for any $v$ satisfy $\|v\|=\|\phi(\theta)\|$, we have
	\begin{equation}\small
	\small
	\begin{aligned}
	&\Bigg\|\lambda\nabla_{\theta}\frac{1}{\delta}\left(\int_{-\delta}^{\delta}\left(\frac{L(\theta+tv)-L(\theta)}{\delta}\right)^{p}dt\right)-\frac{\lambda}{\delta}(\nabla_{\theta} L(\theta)^{T}v)^{p-1}\\
	&(\nabla_{\theta} L(\theta+\frac{p}{p+1}\delta v)
	+\nabla_{\theta} L(\theta-\frac{p}{p+1}\delta v)-2\nabla_{\theta} L(\theta))\Bigg\|<o(\delta\|v\|),
	\end{aligned}
	\label{eq:approximate}
	\end{equation}
	\label{thm:approximate}
	when $p$ is an even number.
\end{proposition}
Proposition \ref{thm:approximate} indicates the gradient respect to parameter $\theta$ can be replaced by a computational efficient term. 
On the other hand, we can achieve a similar result presented as  
\begin{equation}\small\small
\begin{aligned}
&\Bigg\|\nabla_{v}\frac{1}{\delta}\left(\int_{-\delta}^{\delta}\left(\frac{L(\theta+tv)-L(\theta)}{\delta}\right)^{p}dt\right)-\frac{p}{p+1}(\nabla_{\theta} L(\theta)^{T}v)^{p-1}\\
&[\nabla L_{v}(\theta+\frac{p+1}{p+2}\delta v)+\nabla_{v} L(\theta-\frac{p+1}{p+2}\delta v)]\Bigg\|<o(\delta\|v\|).
\end{aligned}
\label{eq:gradient approximate}
\end{equation}
for $v$. The equation gives an approximation for the gradient respect to vector $v$ locate in a product Oblique manifold.
\par 
\begin{figure}[H]\centering
	\begin{subfigure}[b]{0.235\textwidth}
		\includegraphics[width=\textwidth]{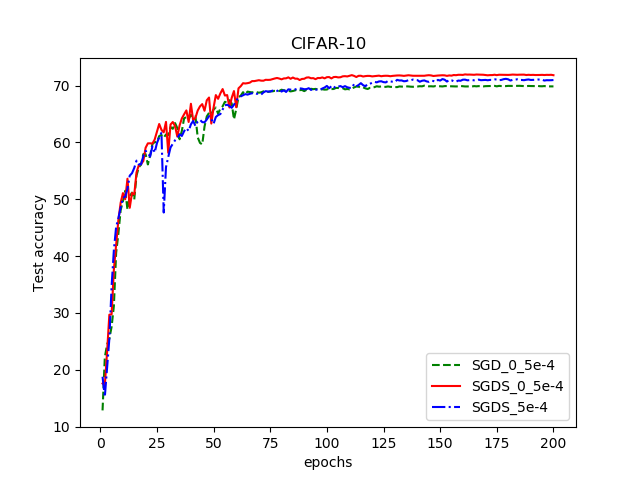}
		\caption{Accuracy of LeNet on CIFAR10}
		\label{fig:lenet}
	\end{subfigure}
	\begin{subfigure}[b]{0.235\textwidth}
		\includegraphics[width=\textwidth]{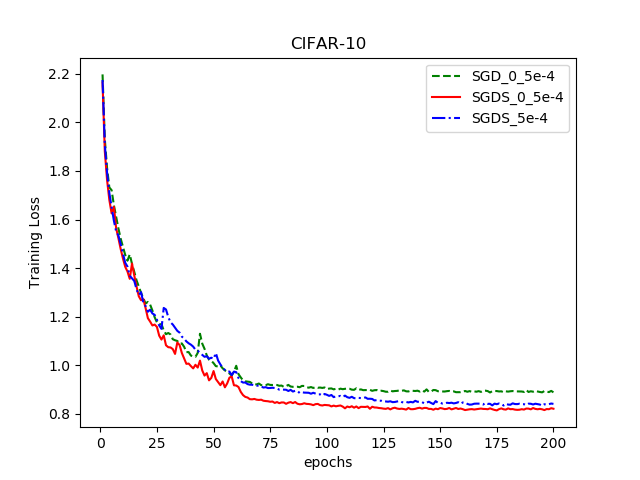}
		\caption{Training loss of LeNet on CIFAR10}
	\end{subfigure}
	\caption{Performance on LeNet}
\end{figure}
We will use the two approximation terms to substitute the original gradients of BN-Sharpness in our algorithm. Finally, we point out that the sharpest direction $v_{k}$ of point $\theta$ in Algorithm \ref{alg1} is produced as
\begin{equation}\small
\begin{aligned}
v^{k}_{i}&={\rm Retr}^{\|\theta_{i}\|}_{v^{k-1}_{i}}\left(P^{\|\theta_{i}\|}_{v^{k-1}_{i}}(h_{2}^{i}(\theta,v^{k-1},\delta,p))\right), i=1,\cdots N_{1}\\
v^{k}_{*} &= {\rm Retr}^{1}_{v^{k-1}_{*}}\left(P^{1}_{v^{k-1}_{*}}(h_{2}^{i}(\theta,v^{k-1},\delta,p))\right),
\end{aligned}
\end{equation}
where $v_{*} = (v_{N_{1}+1}^{T},\cdots,v_{N}^{T})^{T}$.
\section{Experiment}
\noindent
The previous work indicate that SGD with large batch size produces sharp minima while small batch size can avoid sharp minima itself \cite{SharpMinima}. Therefore, in order to finding flat minima under large batch size, we use Algorithm \ref{alg1} to reach such target. We consider our algorithm should have a preferable result comparing to SGD under large batch size. 
We should highlight that the biases gradient $\nabla_{\theta} L\left(\theta+\frac{p}{p+1}\delta v\right)$ and $\nabla_{\theta} L\left(\theta-\frac{p}{p+1}\delta v\right)$ in equation \eqref{eq:approximate} are calculated by different batch data. It's a way of reducing variance which is used in Entropy SGD \cite{entropy_sgd}.
\par
First we test the algorithm with fully batch normalized LeNet \cite{Lenet1} to test the performance for CIFAR10 \cite{cifar}. Update rule is SGD with momentum by setting learning rate as 0.2 and decay it by a factor 0.1 respectively in epoch 60, 120, 160 and momentum parameter as 0.9. We use 10000 batch size, and 5e-4 weight decay ratio for all the three experiments. 
\par
For experiments with regularization term, we clip the gradient of BN-Sharpness by norm with factor 0.1. We use "SGDS-number-decay" to represent the algorithm regularized by BN-Sharpness iterate "number" times of searching sharpest direction and using a weight decay ratio as "decay". For example, "SGDS-5-5e-4" means iterate 5 times of searching $v$ in \ref{def:Lpsharpness} ($K_{1}$=5 in Algorithm \ref{alg1}) and setting weight decay ratio as 5e-4. In addition, the initial point of iteration is set to be the gradient direction like in equation \eqref{eq:sharpnest}. For the experiments with regularized BN-Sharpness, we choose $\lambda$ as 1e-4 which increase by a factor of 1.02 for each epoch. We set $\delta=0.001$, 
and the $p$ is chosen as 2. 
\begin{figure}[H]\centering
	\begin{subfigure}[b]{0.235\textwidth}
		\includegraphics[width=\textwidth]{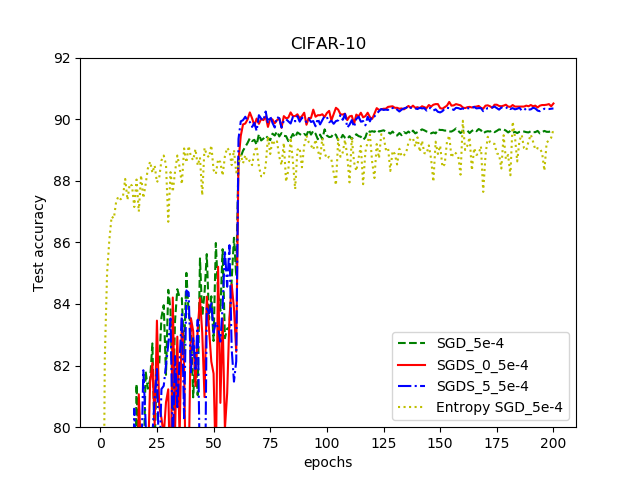}
		\caption{The accuracy of 2048 batch size VGG11 with weight decay on CIFAR10 dataset}
	\end{subfigure}
	\begin{subfigure}[b]{0.235\textwidth}
		\includegraphics[width=\textwidth]{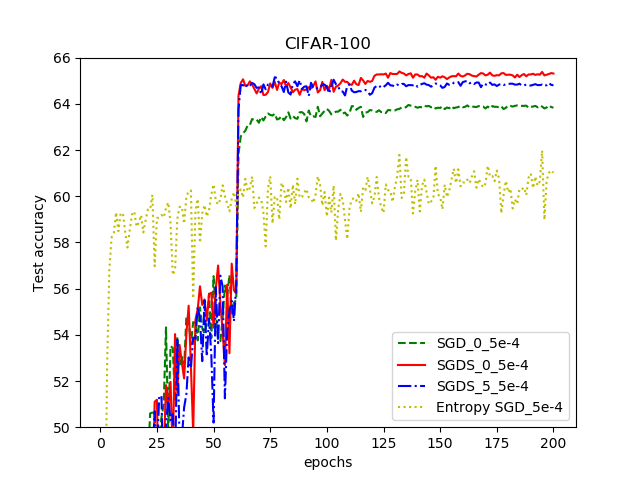}
		\caption{The accuracy of 2048 batch size VGG11 with weight decay on CIFAR100 dataset}
	\end{subfigure}
	\begin{subfigure}[b]{0.235\textwidth}
		\includegraphics[width=\textwidth]{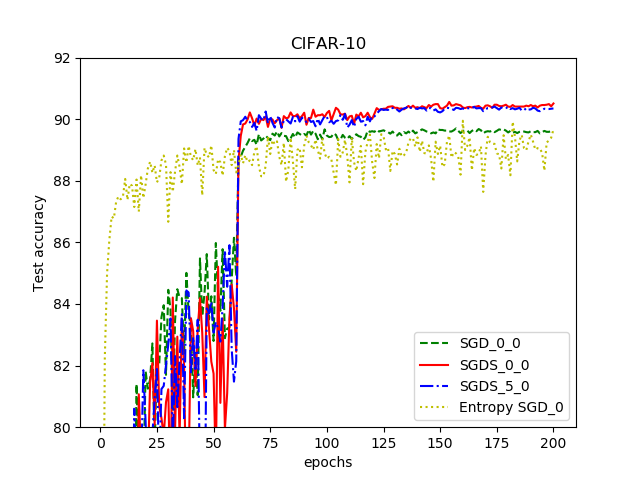}
		\caption{The accuracy of 2048 batch size VGG11 without weight decay on CIFAR10 dataset}
	\end{subfigure}
	\begin{subfigure}[b]{0.235\textwidth}
		\includegraphics[width=\textwidth]{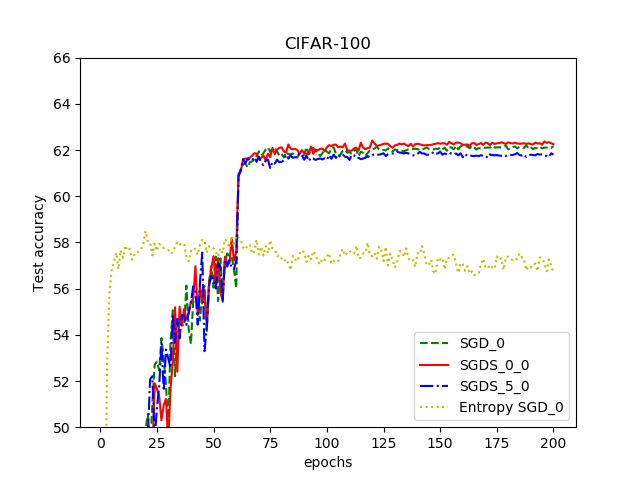}
		\caption{The accuracy of 2048 batch size VGG11 without weight decay on CIFAR100 dataset}
	\end{subfigure}
	\caption{Performance on VGG}
\end{figure}
\begin{table}[htp]
	\centering
	\resizebox{230pt}{30pt}{
		\begin{tabular}{c|cccc}
			\hline
			Algorithm & SGD-5e-4& SGDS-0-5e-4& SGDS-5-5e-4 & E-SGD-5e-4\\
			CIFAR10 & 91.67 & 91.78  & 91.55 & 91.03 \\ 
			CIFAR100 & 69.33 & 69.07 & 68.53 & 67.37 \\
			\hline
			Algorithm & SGD& SGDS-0-0& SGDS-5-0 & E-SGD-0\\
			CIFAR10 & 90.5 & 90.53 & 90.53 & 89.96 \\
			CIFAR100  & 65.5 & 65.45 & 64.82 & 63.45 \\
			\hline
	\end{tabular}}
	\caption{Performance of VGG with batch size as 128.}
\end{table}
\begin{table}[htp]
	\centering
	\resizebox{230pt}{30pt}{
		\begin{tabular}{c|cccc}
			\hline
			Algorithm & SGD-5e-4& SGDS-0-5e-4& SGDS-5-5e-4 & E-SGD-5e-4\\
			CIFAR10 & 89.6 & 90.51 & 90.35 & 89.67 \\ 
			CIFAR100 & 63.83 & 65.31 & 64.8 & 61.02 \\
			\hline
			Algorithm & SGD& SGDS-0-0& SGDS-5-0 & E-SGD-0\\
			CIFAR10 & 88.5 & 88.58 & 88.9 & 87.13 \\
			CIFAR100  & 62.17 & 62.25 & 61.8 & 56.68 \\
			\hline
	\end{tabular}}
	\caption{Performance of VGG with batch size as 2048.}
\end{table}
The accuracy result is referred to Figure \ref{fig:lenet}, where the accuracy results of "SGD", "SGD with sharpness" and "SGD with sharpness iteration=5" are respectively 69.87, 71.84, 70.95. 
We see that our algorithm has a significant promotion on such toy example. Now we test our algorithm on a deeper network VGG11 with bath normalization \cite{VGG}. We respectively test vanilla SGD, SGD with BN-Sharpness regularization and Entropy-SGD (represented as E-SGD) on the network. The experiments can be divided into two groups, large batch size and small batch size. The batch size we used for the two groups of experiments are respectively 128 and 2048.
\par
For SGDS, the $\delta$ in CIFAR10 is 5e-4 and in CIFAR100 is 1e-3, learning rate is 0.2 and decay by a factor 0.1 respectively in epoch 60, 120, 160. The learning rate for SGD is 0.1 and we decay it like SGDS. Other hyper-parameters we used follow the setting in the first experiment of LeNet. 
\par
For Entropy SGD we follow the same hyper parameters in \cite{entropy_sgd} for experiments with 2048 batch size, except for the learning rate. We set learning rate like SGD experiments. But for experiment with 128 batch size, we turn off the drop out and use a 0.01 global learning rate. We also didn't adjust the $\gamma$ accord with iteration times. These adjustments will make Entropy-SGD perform better.
From the results, we conclude that SGD with small batch size can avoid "sharp" minima itself, and regularized by sharpness has little influence. Here we don't pay a lot attention to searching the sharpest direction. This is motivated by reducing computation complexity as well as balancing regularized term and loss function. 
\par
We emphasize that dividing the sharpness target and accuracy target is important. It allows us to dynamically adjust "flat" target which avoid iterates stack in a wide valley when accuracy in a low standard. Actually, if we use the same hyper parameters for Entropy SGD under batch size 128 and 2048. The training accuracy will end in a low level. Even though, Entropy SGD still perform general. We think it suffer from two aspects: First, sampling 20 points($L=20$ in Entropy SGD) for a huge model is not enough, which make it unstable. It will present a high variance result; Second, iterates tend to stack in a inaccurate point, since the loss function may be inexact (See Supplement Material).
\par
Training model with SGD for 600 epochs (adjust learning rate by iterations) can reach the results of SGDS. It's a viewpoint for large batch size SGD: Training longer, generalize better \cite{large1}. However, our method can reach such result under fewer iterations.      
\par
\section{Conclusion}
\noindent
We first prove that tradition definitions of sharpness are insufficient to describe geometrical structure of neural network with batch normalization. Based on that, we propose a scale invariant BN-Sharpness to measure the sharpness of minima. Then, we give an algorithm based on BN-Sharpness. It has computational advantage comparing to existing sharpness based algorithms. 
\par
For all algorithms in order to find "flat" minima (including our algorithm), the training purpose is finding a "flatter" minima generalize better, rather than finding a "flatter" point. Therefore, for such kind of algorithms, setting a large proportion to sharpness target will easily force iterates walk into a "flat" local minimum that generalize poor. Therefore, balancing the sharpness target and loss target is a delicate but meaningful problem.
\par
Finally, we don't particularly discuss the combination of our algorithms with training tricks under large batch size such as \cite{large1}, \cite{large2} and \cite{large3}. In fact, combing them together it's a very interesting topic. 

\bibliographystyle{named}
\bibliography{reference}
\appendix
\onecolumn
\section{Some Properties of BN-sharpness}
	Though BN-sharpness is a substitution of $L^{p}$ norm, we don't konw the precise relationship between $\delta$-sharpness ($L^{\infty}$ norm) and BN-sharpness. Here we theoretically analyse the BN-sharpness for a partially batch normalized neural network.  Before that, we reveal mathematical meaning of $\delta$-sharpness.
	\begin{theorem}
		If $L(\theta)$ is continuous, for given $\theta$ and $\delta$, then $\mathop{\max}_{\theta^{'}\in B_{2}(\theta,\delta)}|L(\theta^{'})-L(\theta)|=\|L(\cdot)\|_{\infty}^{\delta,\theta}$ where $\|\cdot \|_{\infty}^{\delta,\theta}$ is the norm of measurable function space $L^{\infty}(B_{2}(\theta,\delta),\mu)$. $\|L(\cdot)\|_{\infty}^{\delta,\theta}$ is defined as
		\begin{equation}
		\mathop{\inf}_{\mu(E_{0})=0,E_{0}\subset B_{2}(\delta,\theta)}\left(\mathop{\sup}_{\theta^{'}\in B_{2}(\theta,\delta)-E_{0}}|L(\theta^{'})-L(\theta)|\right),
		\label{ess_norm}
		\end{equation}
		where $\mu$ is Lebesgue measure.
		\label{th1:equalivent}
	\end{theorem}
	\begin{proof}
		Obviously, $\mathop{\max}_{\theta^{'}\in B_{2}(\theta,\delta)}|L(\theta^{'})-L(\theta)|\geq\|L(\cdot)\|_{\infty}^{\delta,\theta}$. Suppose that
			\[\mathop{\inf}_{\mu(E_{0})=0,E_{0}\subset B_{2}(\delta,\theta)}\left(\mathop{\sup}_{\theta^{'}\in B_{2}(\theta,\delta)-E_{0}}|L(\theta^{'})-L(\theta)|\right) = \mathop{\sup}_{\theta^{'}\in B_{2}(\theta,\delta)-E^{*}}|L(\theta^{'})-L(\theta)|,\]
			and $\mathop{\max}_{\theta^{'}\in B_{2}(\theta,\delta)}L(\theta^{'})-L(\theta) = L(\theta^{*})-L(\theta)$. If $\theta^{*}$ belongs to $E-E^{*}$, then $\mathop{\max}_{\theta^{'}\in B_{2}(\theta,\delta)}|L(\theta^{'})-L(\theta)|\leq\|L(\cdot)\|_{\infty}^{\delta,\theta}$. On the other hand, if $\theta\in E^{*}$, continuity of $L(\theta)$ implies that for any $\varepsilon>0$ there exist a $\delta^{*}$ satisfy $|L(\theta^{'})-L(\theta^{*})|<\varepsilon$ when $\|\theta^{'}-\theta^{*}\|<\delta^{*}$. Therefore, we have
		\begin{equation}
		|L(\theta^{*})-L(\theta)|-\varepsilon<\mathop{\sup}_{\theta^{'}\in B_{2}(\theta^{*},(\delta-\|\theta^{*}-\theta\|)\wedge\delta^{*})}|L(\theta^{'})-L(\theta)|\leq \mathop{\sup}_{\theta^{'}\in B_{2}(\theta,\delta)-E^{*}}|L(\theta^{'})-L(\theta)|.
		\label{eq:varepsilon}
		\end{equation}
		We conclude that $\mathop{\max}_{\theta^{'}\in B_{2}(\theta,\delta)}|L(\theta^{'})-L(\theta)|\geq\|L(\cdot)\|_{\infty}^{\delta,\theta}$ by the arbitrary of $\varepsilon$ in (\ref{eq:varepsilon}).
	\end{proof}
	\par
	We see that $\delta$-sharpness is actually the $L^{\infty}$ norm of function $L(\cdot)-L(\theta)$. Now, we reveal the relationship between BN-sharpness and $\delta$-sharpness since BN-sharpness is actually a substitution of $L^{p}$ norm.
	\begin{theorem}
		Given a continuous function $L(\theta)$, positive number $\delta$ and point $\theta$, we have $\|L(\theta^{'})\|_{\infty}^{\delta\|\phi(\theta)\|,\theta}\geq 2^{-\frac{1}{p}}\delta\|L(\theta^{'})\|_{p}^{\delta,\theta}$ for any $p<\infty$, and
		\begin{equation}
		\begin{aligned}
		\delta\lim_{p\rightarrow\infty}\|L(\cdot)\|^{\delta,\theta}_{p}=\mathop{\sup}_{t\in[-\delta,\delta]}\mathop{\max}_{v\in\phi(\theta)}L(\theta+tv)-L(\theta)
		\leq\|L(\cdot)\|_{\infty}^{\delta\|\phi(\theta)\|,\theta}.
		\end{aligned}
		\label{eq:inequality}
		\end{equation}
		\label{thm:infty}
	\end{theorem}
	\begin{proof}
			We notice that $\theta+t\delta\|\phi(\theta)\|\in B_{2}(\theta,\|\phi(\theta)\|)$ when $0<t<\delta$ For any $a\neq0$. For any $v$ with $v\in\phi(\theta)$, we have
			\[L(\theta+tv)-L(\theta)\leq \|L(\cdot)\|^{\delta\|\phi(\theta)\|,\theta}_{\infty}.\]
			Then
		\begin{equation}\|L(\cdot)\|^{\delta,\theta}_{p}\leq \frac{1}{\delta^{\frac{1}{p}}}\left(\int_{-\delta}^{\delta}\left(\frac{\|L(\cdot)\|^{\delta\|\phi(\theta)\|,\theta}_{\infty}}{\delta}\right)^{p}dt\right)^{\frac{1}{p}}=\frac{2}{\delta}^{\frac{1}{p}}\|L(\cdot)\|^{\delta\|\phi(\theta)\|,\theta}_{\infty},
		\end{equation}
		which implies
			\[\limsup_{p\rightarrow \infty}\delta\|L(\cdot)\|^{\delta,\theta}_{p}\leq \|L(\cdot)\|^{\delta\|\phi(\theta)\|,\theta}_{\infty}.\]
			Similarly, we can confirm that
			\[\limsup_{p\rightarrow \infty}\delta\|L(\cdot)\|^{\delta,\theta}_{p}\leq\mathop{\sup}_{t\in[-\delta,\delta]}\mathop{\max}_{v\in\phi(\theta)}L(\theta+tv)-L(\theta).\]
			On the other hand, suppose $\mathop{\sup}_{t\in[-\delta,\delta]}\mathop{\max}_{v\in\phi(\theta)}L(\theta+tv)-L(\theta)$ is attained in $(t^{'},v^{'})$. Denote $\theta+t^{'}v^{'}$ as $\theta^{*}$, for any $\varepsilon>0$, choosing $v^{*}$ in BN-sharpness as
			\begin{equation}
			\left(\frac{\theta^{*}_{1}-\theta_{1}}{\|\theta_{1}^{*}-\theta_{1}\|}\|\theta_{1}\|,\cdots,\frac{\theta^{*}_{N_1}-\theta_{N_1}}{\|\theta_{N_{1}}^{*}-\theta_{N_{1}}\|}\|\theta_{N_{1}}\|. \frac{\theta^{*}_{N_{1}+1}-\theta_{N_1+1}}{\|\theta_{N_{1}+1}^{*}-\theta_{N_{1}+1}\|},\cdots,\frac{\theta^{*}_{N}-\theta_{N}}{\|\theta_{N}^{*}-\theta_{N}\|}\right)^{T}
			\end{equation}	
			Let $E_{\varepsilon}$ denote $\{t:|L(\theta+tv^{*})-L(\theta^{*})|\geq\varepsilon,-\delta\leq t\leq\delta\}$. According to the continuity of $L(\theta^{'})$, $\mu(E_{\varepsilon})>0$ where $\mu$ is one dimensional Lebesgue measure.
			\[\|L(\cdot)\|^{\delta,\theta}_{p}\geq \frac{1}{\delta^{\frac{1}{p}}}\left(\int_{E_{\varepsilon}}\left|\frac{L(\theta+tv^{*})-L(\theta)}{\delta}\right|^{p}dt\right)^{\frac{1}{p}}\geq \left(\frac{E_{\varepsilon}}{\delta}\right)^{\frac{1}{p}}\frac{1}{\delta}\left(\mathop{\sup}_{t\in[-\delta,\delta]}\mathop{\max}_{v\in\phi(\theta)}L(\theta+tv)-L(\theta)-\varepsilon\right).\]
			By the arbitrary of $\varepsilon$ we conclude
			\[\liminf_{p\rightarrow \infty}\delta\|L(\cdot)\|^{\delta,\theta}_{p}\geq \mathop{\sup}_{t\in[-\delta,\delta]}\mathop{\max}_{v\in\phi(\theta)}L(\theta+tv)-L(\theta),\]
			which implies $\mathop{\sup}_{t\in[-\delta,\delta]}\mathop{\max}_{v\in\phi(\theta)}L(\theta+tv)-L(\theta)=\delta\lim_{p\rightarrow\infty}\|L(\cdot)\|^{\delta,\theta}_{p}.$ Finally, from the definitions, we can easily conclude the last inequality in equation \eqref{eq:inequality}.
		\end{proof}
	The theorem reveals that BN-sharpness is actually a weaker standard of $\delta\|\phi(\theta)\|$-sharpness.Finally, we use a theorem to reveal the influence of $p$ in $\delta$-$L^{p}$ BN-sharpness.
	\begin{theorem}
		For any $0<p\leq q<\infty$, $\delta>0$ and $\theta$, we have
		\begin{equation}
		\|L(\cdot)\|_{p}^{\delta,\theta}\leq \|L(\cdot)\|_{q}^{\delta,\theta}\cdot(2\delta)^{\left(\frac{1}{p}-\frac{1}{q}\right)}.
		\end{equation}
	\end{theorem}
	\begin{proof}
		Suppose that
			\[\|L(\cdot)\|_{p}^{\delta,\theta} = \left(\frac{1}{\delta}\int_{-\delta}^{\delta}\left|\frac{L(\theta+tv^{*})-L(\theta)}{\delta}\right|^{p}\right)^{\frac{1}{p}}dt,\]
			by Holder inequality and the definition of BN sharpness, we have
			\[
			\begin{aligned}
			\left(\|L(\cdot)\|_{q}^{\delta,\theta}\right)^{p}=\left(\frac{1}{\delta}\int_{-\delta}^{\delta}\left|\frac{L(\theta+tv^{*})-L(\theta)}{\delta}\right|^{p}\right)&\leq \left(\frac{1}{\delta}\int_{-\delta}^{\delta}\left|\frac{L(\theta+tv^{*})-L(\theta)}{\delta}\right|^{p\cdot\frac{q}{p}}\right)^{\frac{p}{q}}\left(\int_{-\delta}^{\delta}1^{\frac{q}{q-p}}dx\right)^{\frac{q-p}{q}}\\
			&\leq\left(\|L(\cdot)\|_{q}^{\delta,\theta}\right)^{\frac{p}{q}}\cdot(2\delta)^{\frac{q-p}{q}}.
			\end{aligned}\]
			Thus, we get the conclusion.
	\end{proof}
	\par
	To summary, these theorems present a link between BN sharpness and original definition of sharpness. Specifically, BN-sharpness is actually a weaker standard of $\delta\|\phi(\theta)\|$-sharpness.
	 Meanwhile we analysis the property of $\|L(\cdot)\|_{p}^{\delta,\theta}$ when $p\rightarrow \infty$. And the last property show that $\|L(\cdot)\|_{p}^{\delta,\theta}$ is monotone with $p$.
	\section{Trace Sharpness}
	In addition, trace of Hessian of a minimum is another popular factor to describe flatness.
	\begin{definition}[Trace-sharpness]
		The trace sharpness $S_{\text{trace-sharpness}}(\theta)$ is defined as $\text{tr}(H(\theta))$, where $H(\theta)$ is the Hessian matrix on point $\theta$.
	\end{definition}
	However, trace sharpness is also ill-posed for BN network.
	\begin{theorem}
		Given a partially batch normalized network, trace sharpness is not scale invariant.
	\end{theorem}
	\begin{proof}
		For a partially batch normalized network, without of generality,  we assume it is batch normalized in the first layer. We choose $\vec{a} = (a_{1},1,\cdots,1)^{T}$, then we see $L(\theta)=L(T_{\vec{a}}(\theta))$ but \begin{equation}
		\nabla^{2}L(T_{\vec{a}}(\theta))=\left[\begin{matrix}
		a_{1}^{-1} I_{n_{1}} &\\
		 & I_{n_{2}}
		\end{matrix}\right]\nabla^{2} l(\theta) \left[\begin{matrix}
		a_{1}^{-1} I_{n_{1}} &\\
		& I_{n_{2}}
		\end{matrix}\right],
		\end{equation}
		where $n_{1}$ is the dimension of $\theta$ and $n_{2}$ is the dimension of $(\theta_{2},\cdots,\theta_{N})^{T}$. Then suppose that
		\begin{equation}
		\nabla^{2} L(\theta)=\left[\begin{matrix}
		A_{1} & * \\
		* & A_{2}\\
		\end{matrix}\right],
		\end{equation}
		where $A_{1},A_{2}$ respectively have same dimension with $I_{n_{1}},I_{n_{2}}. $We have
		\begin{equation}
		\text{tr}(\nabla^{2}L(T_{\vec{a}}(\theta))) = a_{1}^{-2}\text{tr}(A_{1})+\text{tr}(A_{2})\neq\text{tr}(A_{1})+\text{tr}(A_{2})=\text{tr}(\nabla^{2}L(\theta)).
		\end{equation}
		It illustrates that trace-sharpness is not scale invariant.
	\end{proof}
	We see trace sharpness can be upper bounded by $\delta$-sharpness (by Taylor's expansion). Hence, trace sharpness is weaker than $\delta$-sharpness. But BN-sharpness is also weaker than $\delta$-sharpness. We use the next theorem to reveal that BN-sharpness is also a stronger standard compared to trace sharpness.
	\begin{theorem}
		If $\theta$ in BN sharpness is a minimum point, then we have $\|L(\cdot)\|_{p}^{\delta,\theta}\geq \delta\left(\frac{2}{2p+1}\right)^{\frac{1}{p}}\|\phi(\theta)\|^{2}\frac{\text{tr}(H(\theta))}{n}$, where $H(\theta)$ is Hessian of $\theta$
		\label{thm:bounded}
	\end{theorem}
	\begin{proof}
		If $\theta$ is a minimum point, then we have
			\begin{equation}
			\begin{aligned}
			\frac{1}{\delta^{\frac{1}{p}}}\left(\int_{-\delta}^{\delta}\left|\frac{L(\theta+tv)-L(\theta)}{\delta}\right|^{p}dt\right)^{\frac{1}{p}}&=\frac{1}{\delta^{\frac{1}{p}}}\left(\int_{-\delta}^{\delta}\left|\frac{t^{2}v^{T}\nabla^{2}_{\theta}L(\theta)v+o(t^{2})}{\delta}\right|^{p}dt\right)^{\frac{1}{p}}\\
			&\approx\frac{1}{\delta^{1+{\frac{1}{p}}}}\left(\int_{-\delta}^{\delta}t^{2p}\left|v^{T}\nabla^{2}_{\theta}L(\theta)v\right|^{p}dt\right)^{\frac{1}{p}}\\
			&=\delta\left(\frac{2}{2p+1}\right)^{\frac{1}{p}}v^{T}\nabla^{2}_{\theta}L(\theta)v
			\end{aligned}
			\end{equation}
			for any $v$ with $v\in\phi(\theta)$. We have
			\begin{equation}
			\mathop{\sup}_{v:v\in\phi(\theta)}\delta\left(\frac{2}{2p+1}\right)^{\frac{1}{p}}v^{T}\nabla^{2}_{\theta}L(\theta)v\geq \delta\left(\frac{2}{2p+1}\right)^{\frac{1}{p}}\lambda^{\theta}_{\text{max}}\|\phi(\theta)\|^{2}.
			\label{eq:tr_lp_sharpness}
			\end{equation}
			Since $\lambda_{\text{max}}^{\theta}\geq \frac{\text{tr}(H(\theta))}{n}$, we get the conclusion.	
	\end{proof}
	\section{Discussion About Sharpness Based Algorithms}
	There are plentiful sharpness based algorithms aim to producing "flat" minima. We give them a general frame. 
	To understanding easily, we start with Entropy SGD. Entropy SGD set loss function as
	\begin{equation}
	L(\theta) = -\log{\int_{\theta^{'}\in R^{n}}}\exp\left(-f(\theta^{'})-\frac{\gamma}{2}\|\theta^{'} - \theta\|^{2}\right)d\theta{'},
	\label{eq:entropy_sgd}
	\end{equation}
	where $\gamma$ is a positive number, $f(\theta)$ is the original loss function. The $\|\theta^{'}-\theta\|$ term in equation (\ref{eq:entropy_sgd}) gives more weight on the value around $\theta$ which is the motivation of deriving Entropy SGD. However, an interesting phenomenon is that the loss function of Entropy SGD is
	\[
	\begin{aligned}
	L(\theta)&=-\log{\int_{\theta^{'}\in R^{n}}}\exp\left(-f(\theta^{'})-\frac{\gamma}{2}\|\theta^{'} - \theta\|^{2}\right)d\theta{'}
	&= -\log\mathbb{E}_{\theta^{'}}\left[\exp(-f(\theta^{'}))\right],
	\end{aligned}\]
	where $\theta^{'}\sim \mathcal{N}(\theta, \gamma^{-\frac{1}{n}}I_{n})$, $n$ is the dimension of $\theta$. Dividing the term in $\log$ into two parts
	\begin{equation}
	\begin{aligned}
	\mathbb{E}_{\theta^{'}}\exp(-f(\theta^{'})) = \mathbb{E}_{\theta^{'}}\left[\exp(-f(\theta^{'}))\mathbb{I}_{\|\theta^{'}-\theta\|\leq\delta}\right]
	+ \mathbb{E}_{\theta^{'}}\left[\exp(-f(\theta^{'}))\mathbb{I}_{\|\theta^{'}-\theta\|\geq\delta}\right].
	\label{eq:entropy_sgd_divide}
	\end{aligned}
	\end{equation}
	By Jensen inequality, we see that for any $\delta>0$ we have
	\begin{equation}
	\begin{aligned}
	-\log\left(\mathbb{E}_{\theta^{'}}\left[\exp(-f(\theta^{'}))\mathbb{I}_{\|\theta^{'}-\theta\|\leq\delta}\right]\right)
	&\leq -\mathbb{E}_{\theta^{'}}\left[\mathbb{I}_{\|\theta^{'}-\theta\|\leq\delta}\log\left(\exp(-f(\theta^{'}))\right)\right]\\
	&=\mathbb{E}_{\theta^{'}}\left[\mathbb{I}_{\|\theta^{'}-\theta\|\leq\delta}f(\theta^{'})\right].
	\end{aligned}
	\end{equation}
	We conclude that the first part in equation (\ref{eq:entropy_sgd_divide}) is much closer to optimization purpose. However
	\begin{equation}
	\begin{aligned}
	&\mathbb{P}_{\theta^{'}}(\|\theta^{'}-\theta\|\leq\delta)
	&=\mathbb{P}_{\theta^{'}}(\gamma^{-\frac{1}{n}}\chi(n))\int_{0}^{\gamma^{-\frac{1}{n}}\delta}\frac{1}{2^{\frac{n}{2}}\Gamma(\frac{n}{2})}x^{\frac{n}{2}-1}e^{-\frac{x}{2}}dx,
	\end{aligned}
	\end{equation}
	which goes to zero as $n\rightarrow\infty$. That means the first term in the right hand side of equation (\ref{eq:entropy_sgd_divide}) is nearly zero. Hence, loss of Entropy SGD mainly decided by the region outside a neighborhood of $\theta$ which is ambiguous to be used as a loss function.
	\par
	The conclusion is counterintuitive, an intuitive explanation is that the volume of ball in $n$ dimensional space goes to zero when $n$ is large. Since $\log(1+x)\sim x$ when $x\rightarrow 0$, $L(\theta)$ in equation (\ref{eq:entropy_sgd}) is reasonable to be a loss function when $\mathbb{E}_{\theta^{'}}\left[\exp(-f(\theta^{'}))\mathbb{I}_{\|\theta^{'}-\theta\|\geq\delta}\right]$ close to 1.
	\par
	The general frame of algorithm based on local geometrical structure replace loss function $L(\theta)$ with
	\begin{equation}
	\int_{\mathbb{R}^{n}}f(\theta^{'})d\mu_{\theta}(\theta^{'}),
	\end{equation}
	$\mu_{\theta}(\cdot)$ is a finite measure related to $\theta$ defined on $\mathbb{R}^{n}$, it decide the weight of point around $\theta$ contribute to loss. If $\mu_{\theta}(\cdot)$ is absolutely continuous to Lebesgue measure, the loss can be written as
	\begin{equation}
	\int_{\mathbb{R}^{n}}f(\theta^{'})\frac{d\mu_{\theta}(\theta^{'})}{d\theta^{'}}d\theta^{'}.
	\end{equation}
	In Wen et al.(2018) and Wang et al.(2018)'s work, the $\frac{d\mu_{\theta}(\theta^{'})}{d\theta^{'}}$ respectively are $\mathbb{I}_{[\theta-a,\theta+a]}(\theta^{'})$ and $\left(\frac{1}{2\pi\sigma^{2}}\right)^{\frac{n}{2}}\exp\left(-\frac{1}{2\sigma^{2}}\|\theta^{'}-\theta\|^{2}\right)$, where $a,\sigma$ are positive constants related to $\theta$.
	\par
	Under this frame, local algorithm only depends on the measure $\mu_{\theta}(\cdot)$ we choose. In fact, for each $\mu_\theta(\cdot)$, we can define corresponding BN sharpness. The only difference is substituting the norm of $v$ in BN sharpness with the norm induced by $\mu_{\theta}(\cdot)$. In BN sharpness, we use Euclidean norm which induced by the Lebesgue measure constrained in a sphere. However, such a formulation has two inevitable disadvantages, the first is calculating a high dimensional integration. For a NN, the $n$ is usually a huge number, therefore, calculating an accurate loss value of such form is difficult. The other is such loss function may appear to be inappropriate to our optimal objective, such as Entropy SGD.
	\section{A Brief Introduction to Optimization on Manifold}
	A brief summary of optimization on manifold is convert the constraint condition of an optimization problem to a non-constraint problem defined on manifold. The point on such manifold always satisfy constraints. The specific definition of manifold $\mathcal{M}$ can be found in any topology book. Here we only consider matrix manifold(i.e. a subspace of Euclidean space).
	\par
	Suppose $x\in\mathcal{M}$, $\mathcal{M}$ is a manifold, it has a tangent space $T_{x}\mathcal{M}$ which is a linear space but $\mathcal{M}$ may not. The iterates generated by gradient rise on Euclidean space is
	\begin{equation}
	x_{k+1}=x_{k}+\eta\nabla f(x_{k}),
	\label{eq:GD}
	\end{equation}
	$\eta$ is step length. However, the $x_{k}$ in equation (\ref{eq:GD}) may not on manifold because manifold can be a nonlinear space, retraction function ${\rm Retr}_{x}(\eta):T_{x}\mathcal{M}\rightarrow\mathcal{M}$ can fix this problem. Specifically, if $\mathcal{M}$ is $\mathbb{R}^{n}$, the ${\rm Retr}_{x}$ becomes $x+\eta$. We can consider $\eta$ in ${\rm Retr}_{x}(\eta)$ as moving direction of iterating point. Then, the gradient ascent on manifold is given by
	\begin{flalign}
	x_{k+1}={\rm Retr}_{x}\left(\frac{1}{L}{\rm grad}f(x_{k})\right), \label{eq:gdmanifold}
	\end{flalign}
	where ${\rm grad}f(x)$ is Riemannian gradient. Riemannian gradient here is the orthogonal projection of gradient $\nabla f(x)$ to tangent space $T_{x}\mathcal{M}$(rigorous definition need a detailed discussion of geometry structure of manifold which is not the crucial point of this paper). Riemannian gradient is given to decide moving direction on manifold, because $\nabla f(x)$ may not in tangent space $T_{x}\mathcal{M}$ but the moving direction on manifold is only decided by the vector in $T_{x}\mathcal{M}$. All of notations related to manifold can be referred to Boumal et al.\cite{Boumal}. The gradient ascent on manifold helps us find the maximum of function defined on manifold. The convergence theorem of this algorithm has been proved under some extra continuous conditions \cite{Boumal}.
	\section{Missing Proof in the Main Body}
	\begin{proof}[Proof of Theorem \ref{thm:optimization_problem}]
		By Taylor expansion, we have
		\begin{equation}
		\begin{aligned}
		\|L(\cdot)\|_{p}^{\delta,\theta}&=\frac{1}{\delta^{\frac{1}{p}}}\left(\int_{-\delta}^{\delta}\left|\frac{L(\theta+tv)-L(\theta)}{\delta}\right|^{p}dt\right)^{\frac{1}{p}}
		=\frac{1}{\delta^{\frac{1}{p}}}\left(\int_{-\delta}^{\delta}\left|\frac{tv^{T}\nabla_{\theta} L(\theta)+o(t\|\phi(\theta)\|)}{\delta}\right|^{p}dt\right)^{\frac{1}{p}}\\
		&=\frac{1}{\delta^{\frac{1}{p}}}\left(\int_{-\delta}^{\delta}\left|\frac{(tv^{T}\nabla_{\theta} L(\theta))^{p}+o(t\|\phi(\theta)\|)}{\delta^{p}}\right|dt\right)^{\frac{1}{p}}=\frac{1}{\delta^{1+\frac{1}{p}}}\left(\int_{-\delta}^{\delta}t^{p}|v^{T}\nabla_{\theta} L(\theta)|^{p}+o(t\|\phi(\theta)\|)dt\right)^{\frac{1}{p}}\\
		&=\frac{1}{\delta^{1+\frac{1}{p}}}\left(\frac{2\delta^{p+1}}{p+1}|v^{T}\nabla_{\theta} L(\theta)|+o(\delta^{2}\|\phi(\theta)\|)\right)^{\frac{1}{p}}=\left(\frac{2}{p+1}\right)^{\frac{1}{p}}|v^{T}\nabla_{\theta} L(\theta)|(p\geq 1)
		\end{aligned}
		\end{equation}
		when $\delta\rightarrow 0$. Then $\lim_{\delta\rightarrow0}\|L(\cdot)\|_{p}^{\delta,\theta}=\left(\frac{2}{p+1}\right)^{\frac{1}{p}}|v^{T}\nabla_{\theta} L(\theta)|$. Hence for $\lambda>0$ and $\delta\rightarrow0$, $L(\theta)+\lambda(\|L(\cdot)\|_{p}^{\delta,\theta})^{p}\geq L(\theta)$, equality holds when $\nabla_{\theta} L(\theta)=0$ which implies the minimums of $L(\theta)$ also minimize $L(\theta)+\lambda(\|L(\cdot)\|_{p}^{\delta,\theta})^{p}$ when $\delta$ is a small number.
		\end{proof}
\par
	\begin{proof}[Proof of Proposition \ref{thm:approximate}]
		For any $v$ with $v\in\phi(\theta)$, we have
			\[
			\begin{aligned}
			\nabla_{\theta}\frac{1}{\delta}\left(\int_{-\delta}^{\delta}\left(\frac{L(\theta+te)-L(\theta)}{\delta}\right)^{p}dt\right)
			&=\frac{1}{\delta^{2}}\int_{-\delta}^{\delta}p\left(\frac{L(\theta+tv)-L(\theta)}{\delta}\right)^{p-1}(\nabla_{\theta} L(\theta+tv)-\nabla_{\theta} L(\theta))dt\\
			&=\frac{p}{\delta^{p+1}}\int_{-\delta}^{\delta}(t\nabla_{\theta} L(\theta)^{T}v+o(t\|\phi(\theta)\|))^{p-1}(\nabla_{\theta} L(\theta+tv)-\nabla_{\theta} L(\theta))dt\\
			\end{aligned}\]
			Since
			\[
			\begin{aligned}
			\nabla_{\theta} L(\theta+tv)-\nabla_{\theta} L(\theta)&=\nabla_{\theta} L(\theta+tv)-\nabla_{\theta} L(\theta+\varepsilon v)+\nabla_{\theta} L(\theta+\varepsilon v)-\nabla_{\theta} L(\theta)\\
			&=(t-\varepsilon)\nabla_{\theta}^{2}L(\theta+\varepsilon v)v+o(t\|\phi(\theta)\|)+\nabla_{\theta} L(\theta+\varepsilon v)-\nabla_{\theta} L(\theta),
			\end{aligned}\]
			let
			\[h(\theta,\delta,p,v)=\left(\left(t-\frac{p}{p+1}\delta\right)\nabla_{\theta}^{2}L\left(\theta+\frac{p}{p+1}\delta e\right)e+\nabla_{\theta} L\left(\theta +\frac{p}{p+1}\delta v\right)-\nabla_{\theta} L(\theta)\right),\]
			we have
			\[
			\begin{aligned}
			\frac{p}{\delta^{p+1}}\int_{-\delta}^{0}(t\nabla_{\theta} L(\theta)^{T}v&+o(t\|\phi(\theta)\|))^{p-1}(\nabla_{\theta} L(\theta+tv)-\nabla_{\theta}L(\theta))dt\\
			&=\frac{p}{\delta^{p+1}}\int_{-\delta}^{0}\left(t\nabla_{\theta} L(\theta)^{T}v+o(t\|\phi(\theta)\|)\right)^{p-1}(h(\theta,\delta,p,v)+o(t\|\phi(\theta)\|))dt\\
			&=\frac{1}{\delta}(\nabla_{\theta} L(\theta)^{T}v)^{p-1}\left(\nabla_{\theta} L\left(\theta +\frac{p}{p+1}\delta v\right)-\nabla_{\theta}L(\theta)\right)+o(\delta\|\phi(\theta)\|).
			\end{aligned}
			\]
			Here we use the relation
			\[\int_{-\delta}^{0}t^{p-1}\left(t-\frac{p}{p+1}\delta\right)dt=0.\]
			Analogously, we can prove
			\[
			\begin{aligned}
			\frac{p}{\delta^{p+1}}\int_{0}^{\delta}(t\nabla_{\theta} L(\theta)^{T}v&+o(t\|\phi(\theta)\|))^{p-1}(\nabla_{\theta} L(\theta+tv)-\nabla_{\theta} L(\theta))dt\\
			&=\frac{1}{\delta}(\nabla_{\theta} L(\theta)^{T}v)^{p-1}\left(\nabla_{\theta} L\left(\theta-\frac{p}{p+1}\delta v\right)-\nabla_{\theta} L(\theta)\right)+o(\delta\|\phi(\theta)\|)
			\end{aligned}\]
			Finally, we got equation (\ref{eq:approximate}).
	\end{proof}
\end{document}